\documentclass[11pt]{article}

\usepackage{dsfont}
\def\colorful{0}

\usepackage{amsthm,amsfonts,amsmath,amssymb,epsfig,color,float,graphicx,verbatim, enumitem,cleveref}
\usepackage{multirow}

\def\nnewcolor{1}
\ifnum\nnewcolor=1

\fi
\ifnum\nnewcolor=0

\fi

\ifnum\colorful=1
\newcommand{\new}[1]{{\color{red} #1}}

\else
\newcommand{\new}[1]{{#1}}

\fi

\newtheorem{theorem}{Theorem}[section]

\newtheorem{lemma}[theorem]{Lemma}
\newtheorem{informal theorem}[theorem]{Theorem (informal statement)}

\newtheorem{proposition}[theorem]{Proposition}

\newtheorem{claim}[theorem]{Claim}
\newtheorem{fact}[theorem]{Fact}

\newtheorem{remark}[theorem]{Remark}

\theoremstyle{definition}
\newtheorem{definition}[theorem]{Definition}
\newcommand{\eqdef}{\stackrel{{\mathrm {\footnotesize def}}}{=}}

\newcommand{\R}{\mathbb{R}}

\newcommand{\Z}{\mathbb{Z}}
\newcommand{\N}{\mathbb{N}}
\newcommand{\E}{\mathbf{E}}

\newcommand{\eps}{\epsilon}

\newcommand{\pr}{\mathbf{Pr}}
\renewcommand{\Pr}{\mathbf{Pr}}
\newcommand{\poly}{\mathrm{poly}}

\newcommand{\sgn}{\mathrm{sign}}
\newcommand{\sign}{\mathrm{sign}}

\newcommand{\opt}{\mathrm{OPT}}

\author{
Ilias Diakonikolas\thanks{Supported by NSF Medium Award CCF-2107079 and an H.I. Romnes Faculty Fellowship.}\\
University of Wisconsin-Madison\\
{\tt ilias@cs.wisc.edu}\\
\and
Daniel M. Kane\thanks{Supported by NSF Medium Award CCF-2107547.}\\
University of California, San Diego\\
{\tt dakane@cs.ucsd.edu}
}

\title{Statistical Query Lower Bounds for Smoothed Agnostic Learning}

\begin{document}

\maketitle

\begin{abstract}%
We study the complexity of smoothed agnostic learning, 
recently introduced by~\cite{CKKMS24}, 
in which the learner competes with the best classifier in a target class 
under slight Gaussian perturbations of the inputs. 
Specifically, we focus on the prototypical task of 
agnostically learning halfspaces under subgaussian distributions in the smoothed model. 
The best known upper bound for this problem relies on $L_1$-polynomial regression
and has complexity $d^{\tilde{O}(1/\sigma^2) \log(1/\epsilon)}$, 
where $\sigma$ is the smoothing parameter
and $\eps$ is the excess error.
Our main result is a Statistical Query (SQ) lower bound providing formal evidence
that this upper bound is close to best possible. In more detail, we show that 
(even for Gaussian marginals) any SQ algorithm for  
smoothed agnostic learning of halfspaces requires complexity 
$d^{\Omega(1/\sigma^{2}+\log(1/\eps))}$. This is the first non-trivial lower bound 
on the complexity of this task and nearly matches the known upper bound.
Roughly speaking, 
we show that applying $L_1$-polynomial regression 
to a smoothed version of the function is essentially best possible.
Our techniques involve finding a moment-matching hard distribution 
by way of linear programming duality. This dual program corresponds exactly 
to finding a low-degree approximating polynomial to the smoothed version of the 
target function (which turns out to be the same condition required for the $L_1$-polynomial regression to work). 
Our explicit SQ lower bound then comes from proving lower bounds 
on this approximation degree for the class of halfspaces.
\end{abstract}

\setcounter{page}{0}

\thispagestyle{empty}

\newpage

\section{Introduction} \label{sec:intro}

In the agnostic model~\cite{Haussler:92, KSS:94}, the learner is given samples 
from a distribution $D$ on labeled examples in $\R^d \times \{\pm 1\}$, 
and the goal is to compute a hypothesis 
that is competitive with the {\em best-fit} function in a known concept class $\mathcal{C}$. 
In more detail, given $n$ i.i.d. samples $(x^{(i)}, y_i)$ drawn from $D$ and a desired excess error $\eps>0$,
an agnostic learner for $\mathcal{C}$ outputs a hypothesis $h: \R^d \to \{\pm 1\}$ such that 
$\Pr_{(x, y)\sim D}[h(x) \neq y] \leq \opt+\eps$, 
where $\opt \eqdef \inf_{f \in \mathcal{C}}  \Pr_{(x, y)\sim D}[f(x) \neq y]$.  
Importantly, in the agnostic model, no assumptions are made on the observed labels. 
As such, agnostic learning can be viewed as a general framework to study 
supervised learning in the presence of arbitrary misspecification or (potentially adversarial) label noise. 
It is worth recalling that the agnostic model generalizes Valiant's (realizable) PAC model~\cite{Valiant:84}, 
where the labels are promised to be consistent 
with an unknown function in the class $\mathcal{C}$ (i.e., $\opt=0$).

In the original definition of agnostic learning, no assumptions are made on 
the marginal distribution of $D$ on $x$ (henceforth denoted by $D_x$). 
Learning in this {\em distribution-free} agnostic model turns out to be computationally 
challenging. While the model is statistically tractable for all concept classes 
with bounded VC dimension (see, e.g.,~\cite{BEH+:89}), prior work established  
super-polynomial computational lower bounds, 
even for weak agnostic learning of simple concept classes such as halfspaces~\cite{Daniely16,DKMR22, Tiegel23}. 

One approach to circumvent these computational limits involves 
making natural assumptions about the underlying distribution $D_x$ on feature vectors. 
Examples include structured continuous distributions---such as the Gaussian distribution 
or more broadly any log-concave distribution---and discrete distributions---such as 
the uniform distribution on the Boolean hypercube.
While certain computational limitations persist even in this distribution-specific setting, 
see, e.g.,~\cite{DKZ20, GGK20, DKPZ21,DKR23}, 
it turns out that non-trivial positive results are possible. 
Specifically, a line of work has developed efficient learning algorithms with both exact 
(see, e.g.,~\cite{KKMS:08,KOS:08,DKKTZ21})
and approximate (see, e.g.,~\cite{ABL17,DKS18a,DKTZ20c,DKTZ22}) error guarantees 
in the distribution-specific setting. 

A complementary approach---recently introduced in~\cite{CKKMS24} and
the focus of the current paper---involves relaxing the notion of optimality, 
by working in an appropriate {\em smoothed} setting. This approach was 
inspired by a classical line of work in the smoothed analysis of algorithms~\cite{SpielmanT04}. 
For concreteness, we define the smoothed agnostic learning framework below. 

\begin{definition}[Smoothed Agnostic Learning~\cite{CKKMS24}] \label{def:smoothed-agnostic}
Let $\mathcal{C}$ be a Boolean concept class on $\R^d$ and $\mathcal{D}_{x}$ 
be a family of distributions on $\R^d$. We say that an algorithm $\mathcal{A}$ agnostically 
learns $\mathcal{C}$ in the {\em smoothed setting} if it satisfies the following property. 
Given an error parameter $\eps>0$, a smoothing parameter $\sigma>0$ and access to i.i.d.\
samples from a distribution $D$ on $\R^d \times \{\pm 1\}$ such that $D_x \in \mathcal{D}_{x}$, 
$\mathcal{A}$ outputs a hypothesis $h:\R^d \to \{\pm 1\}$ such that with high probability 
it holds $\Pr_{(x, y) \sim D} [h(x) \neq y] \leq \opt_{\sigma}+\eps$, where 
\begin{equation} \label{eqn:opt-sigma}
\opt_{\sigma} \eqdef \inf_{f \in \mathcal{C}} \E_{z \sim \mathcal{N}_d}\left[ \Pr_{(x, y) \sim D} [f(x+\sigma z) \neq y] \right] \;.
\end{equation}
\end{definition}

Some comments are in order. First note that for $\sigma=0$ 
one recovers the standard agnostic model. 
The hope is that when $\sigma>0$ is sufficiently large, 
smoothed agnostic learning becomes computationally tractable 
in settings where vanilla agnostic learning is not. Specifically,
the goal is to characterize (in terms of both upper and lower bounds) 
the computational complexity of smoothed agnostic learning 
for natural concept classes $\mathcal{C}$ and distributions $\mathcal{D}_x$,
as a function of $d, 1/\eps$, and $1/\sigma$.

It turns out that it is possible to develop 
efficient smoothed agnostic learners for natural concept classes $\mathcal{C}$ 
with respect to a broad class of distribution families 
$\mathcal{D}_{x}$---for which vanilla agnostic learning is computationally hard. 
The first positive results were given in~\cite{CKKMS24} 
(who defined the smoothed framework). 
The latter work focused on the natural 
concept class of Multi-Index Models (MIMs)\footnote{A function $f: \R^d \to \{\pm 1\}$ is called a $K$-MIM if there exists a $K$-dimensional subspace $V$ such that for each $x \in \R^d$ 
$f(x)$ only depends on the projection of $x$ onto $V$.} 
that satisfy an additional mild condition (namely, having bounded Gaussian surface area). 
The distribution class $\mathcal{D}_x$ consists of all distributions whose univariate 
projections have subgaussian (or, more generally, strictly sub-exponential) tails. 
In the following discussion,
we focus on the case of subgaussian distributions for concreteness.

In this context,~\cite{CKKMS24} shows that the $L_1$-polynomial regression algorithm~\cite{KKMS:08} 
is a smoothed agnostic learner for $K$-MIMs with Gaussian surface area at most $\Gamma>0$ 
with (sample and computational) complexity $d^{\poly(K, \Gamma, 1/\sigma, 1/\eps)}$. 
This upper bound was recently improved in the same context by~\cite{KW25}, 
who showed that the same algorithm has complexity 
$d^{\poly(K, 1/\sigma, \log(1/\eps))}$---independent of the surface area bound. 
More precisely, the exponent of $d$ in the upper bound of~\cite{KW25} is 
$\tilde{O}(K/\sigma^2) \log(1/\epsilon)$.

Motivated by the aforementioned positive results, here 
we ask the question whether the known upper bounds for smoothed agnostic learning 
can be qualitatively improved. Interestingly, all known algorithmic works 
on the topic (including~\cite{CKKMS24, KW25} and the work of~\cite{KMeka25} studying 
a Boolean-version of smoothing) essentially rely 
on the $L_1$ polynomial regression method---the standard agnostic learner 
in the ``non-smoothed'' setting. The technical novelty of these works lies in the analysis: 
they establish appropriate polynomial approximation results that yield their upper bounds. 

Given this state-of-affairs, it is natural to ask whether the use of $L_1$ polynomial
regression is inherent or whether algorithms with qualitatively better complexity 
are possible in the smoothed setting. More concretely, focusing on the subgaussian case 
for single-index models ($K=1$) and halfspaces in particular, 
the best known upper bound is $d^{\tilde{O}(1/\sigma^2) \log(1/\epsilon)}$. 
Is there a $\poly(d, 1/\eps, 1/\sigma)$ time algorithm? 
Or is the exponential dependence on $1/\sigma^2$ and $\log(1/\eps)$ inherent? 
Motivated by this gap in our understanding, we study the following question:
\begin{center}
{\em What is the complexity of smoothed agnostic learning for halfspaces 
under subgaussian distributions?}
\end{center}
As our main result, we give formal evidence---in the form of a Statistical Query 
lower bound---that the complexity of known algorithms for this problem 
is essentially optimal. 

\medskip

Before we formally state our results, we need basic background 
on the Statistical Query (SQ) Model and the class of Halfspaces.

\paragraph{Statistical Query Model}
Statistical Query (SQ) algorithms are a class of algorithms that are allowed 
to query expectations of bounded functions of the underlying distribution
rather than directly access samples. 
Formally, an SQ algorithm has access to the following oracle.

\begin{definition}[\textsc{STAT} Oracle] \label{def:stat-oracle}
Let $D$ be a distribution over $\R^d \times  \{\pm 1\}$. 
A statistical query is a function $q: \R^d \times \{\pm 1\} \to [-1, 1]$.  We define
\textsc{STAT}$(\tau)$ to be the oracle that given a query $q(\cdot, \cdot)$
outputs a value $v$ such that
$|v - \E_{ (x,y) \sim D}\left[q(x, y)\right]| \leq \tau$,
where $\tau>0$ is the tolerance of the query.
\end{definition}

The SQ model was introduced in~\cite{Kearns:98} 
as a natural restriction of the PAC model 
and has been extensively studied in learning 
theory; see, e.g.,~\cite{FGR+13, FeldmanGV17, Feldman17} and references therein. 
The class of SQ algorithms is fairly broad: a wide range of known algorithmic techniques 
in machine learning are known to be implementable using SQs
(see, e.g.,~\cite{Chu:2006, FGR+13, FeldmanGV17}). 

\medskip

\paragraph{Halfspaces}
A {\em halfspace} is any Boolean-valued function 
$f: \R^d \to \{\pm 1\}$ of the form $f(x) = \sign(w \cdot x - t)$, 
for a weight vector $w \in \R^d$ and a threshold 
$t \in \R$. The algorithmic task of learning halfspaces from labeled
examples is one of the most basic and extensively studied 
problems in machine learning~\cite{Rosenblatt:58, Novikoff:62,
MinskyPapert:68,  FreundSchapire:97, Vapnik:98, CristianiniShaweTaylor:00}. 
Halfspaces are a prototypical family of Single-Index models. 
It is known that even {\em weak} agnostic learning of halfspaces
is computationally hard in the distribution-free setting~\cite{Daniely16}. 
Moreover, in the distribution-specific setting where the distribution is Gaussian,
achieving error $\opt+\eps$ requires complexity $d^{\Omega(1/\eps^2)}$~\cite{DKPZ21,DKR23}.

\subsection{Our Results and Techniques} \label{ssec:results}

We are now ready to state our main result:

\begin{theorem}[Main Result]\label{thm:main-inf}
Any smoothed agnostic learner for the class $\mathcal{H}$ 
of halfspaces on $\R^d$ under the Gaussian distribution, with excess error $\eps$ and smoothing parameter $\sigma>0$, requires SQ complexity 
$d^{\Omega(\log(1/\eps)+ 1/(\sigma+\eps)^{-2})}$. In more detail, 
any SQ algorithm for this task  must either make $2^{d^{\Omega(1)}}$ queries 
or a query of accuracy better than $d^{-\Omega(\log(1/\eps)+ 1/(\sigma+\eps)^{-2})}$.
\end{theorem}

Some comments are in order. First note that for $\sigma \geq \eps$, our SQ lower bound
nearly matches the upper bound of~\cite{KW25}
which holds for all subgaussian distributions. 
That is, the special case of the Gaussian distribution seems to capture the hardness of the 
subgaussian family. In the special case where $\sigma=0$, 
one recovers the known (optimal) SQ lower bound of $d^{\Omega(1/\eps^2)}$~\cite{DKPZ21} 
for vanilla agnostic learning under Gaussian marginals. On the other hand, 
when $0< \sigma = \Omega(1)$, we obtain an SQ lower bound of $d^{\Omega(\log(1/\eps))}$. 

The fact that our SQ lower bound for smoothed agnostic learning of halfspaces 
qualitatively matches the upper bound obtained via $L_1$ polynomial regression 
is not a coincidence. Roughly speaking, we show the following more general result 
(\Cref{thm:LP-general}): for any concept classe $\mathcal{C}$
with natural closure properties 
(including halfspaces and more general MIMs), 
the degree of $L_1$ polynomial approximation to the smoothed function 
$T_{\sigma}f : =  \E_{z \sim \mathcal{N}} \left[ f(x+\sigma z)  \right]$ 
(directly related to the definition of the smoothed optimum \eqref{eqn:opt-sigma}) 
over $f \in \mathcal{C}$ determines the SQ complexity of smoothed agnostic learning. 
Roughly speaking, if the optimal polynomial approximation degree for $\mathcal{C}$ 
is $m^{\ast}$, 
then the SQ complexity of smoothed agnostic learning $\mathcal{C}$
is $d^{\Omega(m^{\ast})}$---matching the complexity of $L_1$-polynomial regression.

It is natural to ask whether one can show stronger SQ lower bounds for 
agnostically learning more general multi-index models in the smoothed setting. Interestingly, we show that this is not possible when the marginal distribution is promised 
to be Gaussian (as in \Cref{thm:main-inf}). Specifically, we show in \Cref{prop:alg-g} 
(given in \Cref{sec:alg-Gaussian}) that $L_1$-polynomial regression is a smoothed agnostic learner for any distribution $(X, y)$ with $X$ Gaussian, with complexity $d^{O(\log(1/\eps)/\sigma^2)}$.

\paragraph{Technical Overview}
The general framework that we leverage to obtain our results follows
a line of work on SQ lower bounds for Non-Gaussian Component Analysis (NGCA) 
and its various generalizations, developed in a series of papers  
over the past decade; see, e.g.,~\cite{DKS17-sq, DKPZ21, DKRS23,DIKR25}.

Specifically, our proof proceeds in two steps: First, we show a generic SQ lower bound
(\Cref{thm:LP-general}) that is applicable to any concept class $\mathcal{C}$ satisfying
natural closure properties. 
In the context of smoothed agnostic learning of halfspaces under the Gaussian 
distribution, this result shows the following: If $f(x) = \sgn(x)$ (the 
univariate sign function), then any SQ algorithm for the problem has 
complexity $d^{\Omega(m)}$, where $d$ is the dimension and $m$ is the 
minimum degree of $L_1$-polynomial approximation with error $O(\eps)$ 
of the smoothed function 
$T_{\sigma}f(x) = \E_{z \sim \mathcal{N}} \left[ f(x+\sigma z)  \right]$. 
The proof of  our generic lower bound proof follows the overall strategy 
of~\cite{DKPZ21} (who prove our bound in the special case of $\sigma=0$). 
The basic idea is to construct a moment-matching distribution on $y$ 
so that $\opt_{\sigma} + \eps$ 
is still somewhat less than $1/2$. 
Thus, any algorithm that produces a hypothesis with error $\opt_{\sigma}+\eps$ would allow 
one to distinguish between this $y$ and a $y$ that is independent of $x$. 
We then show that this distinguishing (testing) problem is SQ-hard, as it corresponds to a hard instance of conditional NGCA~\cite{DIKR25}. 
To construct the desired moment-matching distribution on $y$, 
we use LP duality and show that the relevant dual program 
corresponds exactly to finding a low-degree $L_1$ polynomial 
approximation of $T_{\sigma} f$.

The second step of our proof involves proving explicit degree lower bounds for $T_{\sigma} f$. Here we establish a degree lower bound of 
$c (\log(1/\eps)+ 1/(\sigma+\eps)^2)$ for some universal constant $c>0$\footnote{Note that the second term is roughly $1/\sigma^2$, as long as $\sigma \geq \eps$. For $\sigma=0$, degree $O(1/\eps^2)$ is known to suffice.}. 

To establish a degree lower bound of $\log(1/\eps)$ (\Cref{lem:logeps}) 
we proceed as follows. We start by noting the identity $T_{\sigma}f = U_{a}f$, 
for the function $f(x) = \sgn(x)$, where $U_a$ is the the Ornstein-Ulenbeck (OU) operator and
$a = 1/\sqrt{1+\sigma^2}$. An explicit analysis, using properties of Hermite polynomials 
and the OU operator, allows us to show that the $L_2$ error of any degree-$k$ polynomial 
approximation to $T_{\sigma} f$ must be at least $2^{-O(k)}$. In fact, we can even certify this error by noting that the inner product of $T_{\sigma} f - p$ with an appropriate Hermite polynomial must be large. 
The next step is to compare the $L_1$ and $L_2$ polynomial approximations, which 
we carry out by using results about the concentration of $p$ and this Hermite polynomial. 
As a result, we show that the $L_1$ error of any degree-$k$ polynomial 
approximation to $T_{\sigma} f$ must be $2^{-O(k)}$ as well 
(for a different universal constant hidden in the $O()$.)

To establish the $L_1$ degree lower bound of $(\sigma+\eps)^{-2}$ (\Cref{lem:sigma-sq}), 
we prove the following new structural 
result (\Cref{prop:first-mm-ltfs}) that may be of broader interest: 
there exists a $k$-wise independent family of univariate Gaussians 
(that we explicitly construct) so that their sum mod $1$ is close 
to any specified distribution. The sum of these Gaussians divided by $\sqrt{k}$ will be a 
distribution that matches $k$ moments with the standard Gaussian but has a very different 
distribution when taken modulo $1/\sqrt{k}$.
By letting $X \mid y=1$ be such a distribution 
missing a large fraction of the possible values modulo $1/\sqrt{k}$, 
it is easy to see that for some threshold $f$ 
that the $|T_{\sigma} f(X) - T_{\sigma} f(G)| \gg \eps$ 
so long as $\eps, \sigma \ll 1/\sqrt{k}$. 
This gives us an example of a moment-matching distribution 
$(X,y)$ with $\opt_{\sigma}$ substantially less than $1/2$, 
which can be plugged into our generic SQ lower bound.

Regarding our upper bound for the smoothed Gaussian case (\Cref{prop:alg-g}), 
it suffices to prove an upper bound on the degree of $L_2$ polynomial 
approximation of $T_{\sigma} f$. 
To do this, we note that $T_{\sigma} f = U g$ 
for $g(x) = f(\sqrt{1+\sigma^2}x)$ and $U$ is the relevant element of the Ornstein-Ulenbeck semigroup. 
But if $g$ has Hermite transform $g= \sum g^{[k]}$, 
then $U g$ is $\sum g^{[k]}(1+\sigma^2)^{-k/2}$. 
Since $\| g^{[k]} \|_2 \leq \|g\|_2 \leq 1$, taking the first 
$O(\log(1/\eps)/\sigma^2)$ terms suffices to give error $\eps.$

\section{Preliminaries} \label{sec:prelims}

\paragraph{Notation}
For $n \in \Z_+$, we denote $[n] \eqdef \{1, \ldots, n\}$.
We typically use small letters to denote random variables when the underlying distribution is clear from the context.
We use $\E[x]$ for the expectation of the random variable $x$
and $\pr[\mathcal{E}]$ for the probability of event $\mathcal{E}$. 
Let $\mathcal{N}$ denote the standard univariate Gaussian distribution
and $\mathcal{N}_d$ denote the standard $d$-dimensional Gaussian distribution.
We use $\phi_d$ to denote the pdf of $\mathcal{N}_d$.

Small letters are used for vectors and capital letters are used for matrices.
Let $\|x\|_2$ denote the $\ell_2$-norm of the vector $x \in \R^d$.
We use $u \cdot v $ for the inner product
of vectors $u, v \in \R^d$.
For a matrix $P \in \R^{m \times n}$, let $\|P\|_2$ denote
its spectral norm and $\|P\|_F$ denote its Frobenius norm.
We use $I_d$ to denote the $d \times d$ identity matrix.

\paragraph{Hermite Analysis} \label{ssec:hermite}
Consider $L_2(\R^d, \mathcal{N}_d)$, the vector space of all
functions $f : \R^d \to \R$ such that $\E_{x \sim \mathcal{N}_d}[f(x)^2] <\infty$.
This is an inner product space under the inner product
$\langle f, g \rangle = \E_{x \sim \mathcal{N}_d} [f(x)g(x)] $.
This inner product space has a complete orthogonal basis given by
the \emph{Hermite polynomials}.
In the univariate case, we will work with normalized Hermite polynomials
defined below.

\begin{definition}[Normalized Probabilist's Hermite Polynomial]\label{def:Hermite-poly}
For $k\in\N$, the $k$-th \emph{probabilist's} Hermite polynomial
$He_k:\R\to \R$ is defined as
$He_k(t)=(-1)^k e^{t^2/2}\cdot\frac{\mathrm{d}^k}{\mathrm{d}t^k}e^{-t^2/2}$.
We define the $k$-th \emph{normalized} probabilist's Hermite polynomial
$h_k:\R\to \R$ as
$h_k(t)=He_k(t)/\sqrt{k!}$.
\end{definition}

\noindent Note that for $G\sim \mathcal{N}$ we have $\E[h_n(G)h_m(G)] = \delta_{n,m}$,
and $ \sqrt{m+1} h_{m+1}(t) = t h_m(t) - h'_m(t)$.

\new{We will use various well-known properties of these
polynomials in our proofs.}

\paragraph{Smoothing Operators}
We will also need the following smoothing operators:
\begin{itemize}[leftmargin=*]
\item For a function $f: \R^k \to \R$ and $\sigma>0$, we define the operator
$T_{\sigma}f(x) = \E_{z \sim \mathcal{N}_k} \left[ f(x+\sigma z)  \right]$. This operator is motivated
by the definition of our smoothing model.
\item  For a function $f: \R^k \to \R$ and $\rho \in [0, 1]$, the Ornstein-Uhlenbeck (OU) operator is defined 
as follows $U_{\rho}f(x) = \E_{z \sim \mathcal{N}_k} \left[ f( \rho x+ \sqrt{1-\rho^2} z)  \right]$. This operator
arises in some of our analysis.
\end{itemize}
\new{We note that for homogeneous functions these two operators are essentially equivalent
up to scaling. Specifically, we have that $T_{\sigma} f = U_a g$, 
where $a = 1/\sqrt{1+\sigma^2}$ and $g(x) = f(\sqrt{1+\sigma^2}x)$.} 

We will use the following standard properties of the OU operator:

\begin{fact}[See, e.g., \cite{Bog:98} and \cite{ODonnellbook} (Chapter 11)]
We have the following:
\begin{enumerate}
\item (Self-adjointness) For $f, g\in L_2(\mathcal{N})$ and $\rho\in(0,1)$,
$\E_{x \sim \mathcal{N}}[ (U_{\rho} f(x)) g(x)] = \E_{x \sim \mathcal{N}}[ (U_{\rho} g(x)) f(x)].$

\item (Hermite eigenfunctions) For $i\in\Z_+$, $U_{\rho} h_i(z)=\rho^i h_i(z)$.
\end{enumerate}
\end{fact}

\paragraph{Analytic Facts}
For a polynomial $p: \R^k \to \R$, we consider the random variable $p(x)$,
where $x \sim \mathcal{N}_k$. We will use $\|p\|_r$, for $r \geq 1$, 
to denote the $L_r$-norm
of the random variable $p(x)$, i.e., $\|p\|_r \eqdef \E_{x \sim  \mathcal{N}_k} [|p(x)|^r]^{1/r}$.

We will need the following analytic and probabilistic facts.
We first recall the following moment bound for low-degree
polynomials, which is equivalent to the hypercontractive
inequality \cite{Bon70,Gross:75}:
\begin{fact}[Hypercontractive Inequality] \label{fact:hc}
Let $p: \R^n \to \R$ be a degree-$d$ polynomial and $q>2$.  Then
$\|p\|_q \leq (q-1)^{d/2} \|p\|_2$.
\end{fact}

As a simple corollary, we can relate the $L_1$-norm of a polynomial with the $L_2$-norm.
Specifically, we obtain the following folklore fact:

\begin{fact} \label{fact:2norm}
Let $p: \R^n \to \R$ be a degree-$d$ polynomial. Then
$\|p\|_2 \leq  2^{d/2} \|p\|_1$.
\end{fact}
\begin{proof}
Using the hypercontractive inequality for $q=3$ gives $\|p\|_3 \leq 2^{d/2} \|p\|_2$.
By the Cauchy-Schartz inequality, we have that $\|p\|_2 \leq \|p\|_1^{1/2} \|p\|_3^{1/2}$.
Combining the two bounds gives the proof.
\end{proof}

\section{Main Results} \label{sec:results}

The structure of this section is as follows: In \Cref{ssec:generic-SQ}, we prove our generic 
SQ lower bound (\Cref{thm:LP-general}) for smoothed agnostic learning of any concept class 
$\mathcal{C}$ satisfying appropriate closure properties. Roughly speaking, this result shows 
that the SQ complexity of the task is $d^{\Omega(m)}$, where $m$ is the optimal degree of $L_1$-polynomial approximation of a smoothed version of the target functions. 
In \Cref{ssec:apps}, we prove explicit degree lower bounds for the class of halfspaces, 
thereby establishing \Cref{thm:main-inf}. 

\subsection{Generic SQ Lower Bound} \label{ssec:generic-SQ}

\begin{theorem}[Generic SQ Lower Bound for Smoothed Agnostic Learning] \label{thm:LP-general}
Let $\mathcal{C}$ be a class of Boolean functions on $\R^d$
that for some $f:\R^k \to \{\pm 1 \}$ contains all functions $f(P(x))$ any projection
$P: \R^d \to \R^k$ (i.e., a linear map with $P P^{\top} = I_d$).
Suppose that for some positive integer $m$ the following holds for some $\eps, \sigma>0$:
for any degree at most $m$ polynomial $p: \R^k \to \R$ 
\begin{equation} \label{eqn:smoothed-lb}
\E_{x \sim \mathcal{N}_k} [|T_{\sigma}f(x) - p(x)|] > \eps+2\tau \;,
\end{equation}
where $\tau \eqdef \Theta_{k,m}(d)^{-m/5}$ with the implied constant sufficiently small. 
Then any SQ algorithm that $\sigma$-smoothed-agnostically learns $\mathcal{C}$ to excess
error $\eps$ must either make $2^{d^{\Omega(1)}}$ queries 
or a query of accuracy better than $\tau$.
\end{theorem}

To prove Theorem~\ref{thm:LP-general}, we need the following proposition:

\begin{proposition} \label{prop:moment-matching}
In the context of Theorem~\ref{thm:LP-general},
there exists a distribution $(X,Y)$ on $\R^k \times \{\pm 1\}$ such that:
\begin{enumerate}[leftmargin=0.8cm]
\item[(i)] The $X$-marginal is a standard Gaussian $\mathcal{N}_k$.
\item [(ii)] The expectation of $Y$ is $0$.
\item[(iii)] The conditional distributions $X \mid Y=1$ and $X \mid Y=-1$
each match their first $m$ moments with the standard Gaussian $\mathcal{N}_k$.
\item[(iv)] If $Z$ is an independent standard Gaussian on $\R^k$, then
$\Pr[f(X+\sigma Z) \neq Y] < 1/2 - \eps - 2\tau.$
\end{enumerate}
\end{proposition}

We start by showing how Proposition~\ref{prop:moment-matching}
can be used to prove Theorem~\ref{thm:LP-general}.

\smallskip

\begin{proof}[Proof of Theorem~\ref{thm:LP-general}]
We will use the distribution $(X,Y)$, whose existence is established in 
Proposition~\ref{prop:moment-matching}, 
to produce a \new{conditional} non-Gaussian component analysis construction~\cite{DIKR25}.
We begin by defining a hard ensemble.
In particular, for a random 
projection $P: \R^d \to \R^k$,
we define the distribution on $\R^d \times \{\pm 1\}$
as $(P^{\top}(X)+X', Y)$, where $(X,Y)$ is the distribution 
from Proposition~\ref{prop:moment-matching}
and $X'$ is a standard Gaussian supported on the subspace $V^{\perp}$
where $V$ is the image of $P^{\top}$. Notice that this gives an instance of a Relativized 
Hidden Subspace Distribution for $A=(X,y)$ and $U=P^\top$ 
from Definition 3.1 of \cite{DIKR25}.

Furthermore, as $X|Y=1$ and $X|Y=-1$ both match $m$ moments with a standard Gaussian, 
Theorem 3.3 of \cite{DIKR25} implies that any Statistical Query algorithm 
distinguishing $(x,Y)$ (with randomly chosen $P$) from $(x,Y')$ 
(where $Y'$ is an independent, uniform $\{\pm 1\}$ random variable) 
must either use a query of accuracy better than $\tau$ 
or use $2^{d^{\Omega(1)}}$ queries.

On the other hand, we claim that given a smoothed-agnostic learner for $\mathcal{C}$ with excess error $\eps$,
an additional statistical query of tolerance $\tau$ would allow us to correctly solve this distinguishing problem.
This is because the learned hypothesis $h: \R^d \to \{\pm 1\}$ would have the property that 
\begin{eqnarray*}
\Pr[h(x) \neq Y] &\leq& \opt_{\sigma} + \eps \\
                         &\leq&  \Pr[f(P(x+ \sigma z)) \neq Y] + \eps \\
&=& \Pr[f(X+\sigma Z) \neq Y] + \eps \\
&<& 1/2 - 2\tau \;.
\end{eqnarray*}
On the other hand, since $Y'$ is independent of $x$, we have
$\Pr[h(x) \neq Y'] = 1/2.$
As a single statistical query can distinguish between
$\Pr[h(x) \neq Y]$ being $1/2$ or less than $1/2-2\tau$,
this completes the reduction and the proof of Theorem~\ref{thm:LP-general}.
\end{proof}

To prove Proposition~\ref{prop:moment-matching}, we use linear programming duality 
techniques, similar to those from~\cite{DKPZ21} 
(which corresponds to the $\sigma=0$ special case of our proof below).

\smallskip

\begin{proof}[Proof of Proposition~\ref{prop:moment-matching}]
To begin with, we note that finding a distribution $(X,Y)$
is  equivalent to specifying the function $g(x) = \E[Y \mid X=x]$.
In particular, note that $g$ must be a function from $\R^k$ to $[-1,1]$. 
\new{We now translate the desired properties of $(X,Y)$
to corresponding properties of the function $g$.}

We start with the following simple claim:
\begin{claim}\label{clm:reform1}
Conditions (ii) and (iii) on $(X,Y)$ together are equivalent to the condition
that $g$ satisfies $\E[g(X)p(X)] = 0$ for any degree at most $m$ polynomial $p: \R^k \to \R$.
\end{claim}
\begin{proof}
One direction of the equivalence is straightforward.
Suppose that conditions (ii) and (iii) on $(X, Y)$ are satisfied. Since
$g(x) = \E[Y \mid X=x]$, it follows 
that for any degree at most $m$ polynomial $p$
we have that
\begin{align*}
\E[g(X)p(X)] & = \E[Yp(X)]\\
& = \pr(Y=1)\E[p(X)|Y=1] - \pr(Y=-1)\E[p(X)|Y=-1]\\
& = (\E[p(X)|Y=1] - \E[p(X)|Y=-1])/2\\
& = 0.
\end{align*}

For the opposite direction, suppose that
$\E[g(X)p(X)] = 0$ for any degree at most $m$ polynomial $p: \R^k \to \R$.
Setting $p \equiv 1$ gives condition (ii). 
Note that the latter implies that  $\Pr[Y=1] = \Pr[Y=-1] = 1/2$.
Given this, we have that for any degree at most $m$ polynomial $p$, it holds
\begin{eqnarray*}
\E[p(X) \mid Y=1] = \frac{\E[p(X) \mathbf{1}\{Y=1\}]}{\Pr[Y=1]}
= 2 \, \E[p(X) \mathbf{1}\{Y=1\}] = \E[p(X)(g(X)+1)] = \E[p(X)] \;.
\end{eqnarray*} 
An identical argument can be used for $\E[p(X) \mid Y=-1]$,
giving (iii). This completes the proof of Claim~\ref{clm:reform1}.
\end{proof}

We finally translate condition (iv).
We can write
\begin{eqnarray*}
\Pr[f(X+\sigma Z) \neq Y] &=& (1-\E[f(X+\sigma Z) Y])/2 \\
&=& (1-\E[(T_{\sigma} f)(X) Y])/2 \\
&=& (1-\E[(T_{\sigma} f)(X) g(X)])/2 \;.
\end{eqnarray*} 
Thus, condition (iv) is equivalent to the condition
\begin{equation} \label{eqn:T-cond}
\E[T_{\sigma} f(X) g(X)] > \eps + 2\tau \;.
\end{equation}
In summary, to prove the proposition, it suffices to find a function $g: \R^k \to \R$ such that
\begin{itemize}
\item[(A)] $\|g \|_{\infty} \leq 1$.
\item[(B)] $\E_{X \sim \mathcal{N}_k}[g(X)p(X)] = 0$ for all degree at most $m$ polynomials $p$.
\item[(C)] $\E_{X \sim \mathcal{N}_k}[(T_{\sigma}) f(X) g(X)] > \eps + 2\tau$.
\end{itemize}
We note that constraints (A)-(C) define an infinite linear program.
By a generalization of standard LP duality to this setting, 
this linear program is feasible if and only if the corresponding dual program is infeasible. This would follow immediately from the standard LP duality result 
were it a finite linear program. In our case, we need to use some slightly 
more complicated analytic tools, and follows directly from 
\new{Corollary 55} of~\cite{DKPZ21}.

This asks if there is a linear combination of constraints that reach a contradiction.
In particular, the constraints in part (A) combine to form constraints of the form
$\E[g(X)a(X)] \leq \|a\|_1$ for all $a$.  
The constraints in part (B) in general are of the form $\E[g(X)p(X)] = 0$.
The constraints in part (C) are just multiples of the single constraint
that $\E[T_{\sigma} f(X) g(X)] > \eps+2\tau$.

Thus, a solution to the dual program must consist of a function $a$,
a polynomial $p$ and a negative real number $c$ such that
$$a(X) + p(X) + c \, T_{\sigma} f(X)$$
is identically $0$
and $$\|a\|_1 + c(\eps+2\tau) < 0 \;.$$
By homogeneity, we can take $c=-1$, so $a(X) = T_{\sigma} f(X) - p(X)$
and $\|a\|_1 < \eps+2\tau.$ In particular, there is a solution to the dual program
if and only if there is a degree at most $m$ polynomial $p$
such that for $\|T_{\sigma} f(X) - p(X)\|_1 < \eps+2\tau$.
However, this is ruled out by assumption.
Therefore, our primal LP must have a solution, proving Proposition~\ref{prop:moment-matching}.
\end{proof}

\subsection{Smoothed Agnostic Learning of Halfspaces: Proof of \Cref{thm:main-inf}} \label{ssec:apps}

Our main application is for smoothed agnostic learning of halfspaces,
when the distribution $D_x$ is the standard Gaussian. This corresponds to the
case $k=1$ and $f(x) = \sign(x)$ in Theorem~\ref{thm:LP-general}.

\Cref{thm:main-inf} follows directly from Theorem~\ref{thm:LP-general} 
and the following result, which is established in this section. 

\begin{theorem}[Smoothed $L_1$-Degree Lower Bound for the Sign Function] \label{thm:LTF-degree-lb}
Let $f(x) = \sign(x)$. For $\sigma, \eps>0$, we have that for any polynomial $p: \R \to \R$
of degree at most a small constant multiple of 
$m = \min \{(\sigma+\eps)^{-2} + \log(1/\eps)\}$, it holds that
$\E_{x \sim \mathcal{N}} [|T_{\sigma}f(x) - p(x)|] > \eps$.
\end{theorem}

We start by proving a degree lower bound of
$(\sigma+\eps)^{-2}$.

\begin{lemma} \label{lem:sigma-sq}
Theorem \ref{thm:LTF-degree-lb} holds instead taking $m=(\sigma+\eps)^{-2}$.
\end{lemma}

The key technical statement that we require
is in the following proposition.

\begin{proposition} \label{prop:first-mm-ltfs}
For $k > 0$ and any subset $S$ of $[0,1]$, there exists a probability distribution $X$ on $\R$ so that:
\begin{itemize}
\item[(A)] $X$ matches its first $k$ moments with the standard Gaussian.
\item[(B)] $X(x)/\phi(x) \leq (1/|S|)$, where $|S|$ denotes the measure of $S$.
\item[(C)] For some constant $C>0$, all but $2^{-k}$ of the probability
mass of $X$ is supported on points
$x$ so that $\mathrm{FractionalPart}(x (\sqrt{C k})) \in S$.
\end{itemize}
\end{proposition}
\begin{proof}
Let $C$ be a large enough integer.
We will let $X = 1/\sqrt{Ck} \sum_{i=1}^{Ck} X_i$,
where the $X_i$'s are non-independent Gaussians.
In particular, you can write the standard Gaussian $G$ as a mixture of $U([0,1])$
with some other distribution $E$. In particular,
$G = \alpha U([0,1]) + (1-\alpha)E$. We sample from $X$ in the following way:
\begin{enumerate}
\item[(1)] Sample $y_i$ iid from $\{0,1\}$
with $y_i=1$ with probability $\alpha$.
\item[(2)] For each $y_i=0$, set $X_i$ to be an independent sample from $E$.
\item[(3)] If $\sum y_i < k+1$, set each $X_i$ with $y_i = 1$ to be an independent sample from $U([0,1])$.
\item[(4)] Otherwise, set the $X_i$'s with $y_i=1$ to be independent samples from $U([0,1])$
conditioned on the sum of all $X_i$'s mod $1$ lying in $S$.
\item[(5)] Let  $X = 1/\sqrt{Ck} \sum_{i=1}^{Ck} X_i$.
\end{enumerate}
To analyze this, we begin by noting that in (4)
if instead the $y_i$'s are taken to be independent $U([0,1])$ random variables,
then the $X_i$'s are independent Gaussians, so X is a standard Gaussian.

To prove (A), we note that in step (4) the $y_i$'s are instead
each uniformly distributed and $k$-wise independent,
so they have the same first $k$ moments as if they were independent;
and thus $X$ has the same first $k$ moments as a standard Gaussian.
For (B), we note that we obtain $X$ by taking a sampling process
that generates a standard Gaussian and in step (4) conditioning
on an event of probability $|S|$. Finally to prove (C), we note that
if $\sum y_i \geq k+1$, that $\mathrm{FractionalPart}(x/(\sqrt{C k})) \in S$.
Since the sum of the $y_i$'s is distributed is $Bin(\alpha,Ck)$,
this happens with probability at least $1-2^{-k}$,
when $C$ is large enough. This completes the proof of Proposition~\ref{prop:first-mm-ltfs}.
\end{proof}

\begin{remark} 
{\em Conditions (B) and (C) in the statement of Proposition~\ref{prop:first-mm-ltfs} 
can be replaced by the statement that $X$ is $2^{-\Omega(k)}$ close 
in total variation distance to the standard Gaussian 
conditioned on the fact that $\mathrm{FractionalPart}(x/(\sqrt{C k})) \in S$. 
This statement may be useful for other applications.
}
\end{remark}

\begin{proof}[Proof of Lemma~\ref{lem:sigma-sq}]
Given Proposition~\ref{prop:first-mm-ltfs},
let $k$ be a small enough constant multiple of $(\sigma+\eps)^{-2}$
and let $S = [1/2,1]$ and consider the corresponding distribution $X$.
Let $t = 0$ and let $t' = 1/(2 \sqrt{Ck})$ with $C$ as in Proposition \ref{prop:first-mm-ltfs}. 
Let $f(x) = \sgn(x-t)$ and $f'(x) = \sgn(x-t')$.
Let $g = T_{\sigma} f$ and $g' = T_{\sigma} f'$.
We note that $t$ and $t'$ differ by $\Omega(k^{-1/2})$
which is a large constant times $(\sigma+\eps)$.
From this it is easy to see that
$\E_{G \sim \mathcal{N}}[g(G)] - \E_{G \sim \mathcal{N}}[g'(G)]$ is at least a constant times the probability that $G$ lies in $[(2t+t')/3,(t+2t')/3]$, which
is at least a large constant times $(\sigma+\eps)$. 

On the other hand, it is not hard to see that $\E[g(X)] - \E[g'(X)] = O(\sigma)$,
because discounting the region between $t$ and $t'$
(a region where $X$ assigns exponentially small 
mass)
the $L_1$-norm (wrt $G$) of $g-g'$ is $O(\sigma)$.

However, since $X$ is bounded by $2G$,
\new{where $G \sim \mathcal{N}$, (by property (B))}
and since it matches its first $k$ moments with a Gaussian \new{(by property (A))},
if $p$ is a degree at most $k$ polynomial with $\|g-p\|_1 < \delta$,  
then
\begin{eqnarray*}
\E[g(X)] &=& \E[p(X)] + \E[(g-p)(X)] \\
             &=&  \E[p(G)] + O(\delta)  \\
             &=&  \E[g(G)] + \E[(g-p)(G)] + O(\delta) \\
             &=& \E[g(G)] + O(\delta) \;.
\end{eqnarray*}
A similar bound holds for $g'$.

However, if such approximations exist for both $g$ and $g'$, we have that
$$\E_{G \sim \mathcal{N}}[g(G)-g'(G)] = \E[g(X)-g'(X)] + O(\delta) \;,$$
which by the above is a contradiction
unless $\delta$ is at least a large constant times $\sigma+\eps$. 
This shows that there is a $T_{\sigma}$
LTF that is not well approximable by degree $k$ polynomials.
This completes the proof of Lemma~\ref{lem:sigma-sq}.
\end{proof}

\begin{lemma} \label{lem:logeps}
Theorem \ref{thm:LTF-degree-lb} holds instead taking $m=\log(1/\eps)$, 
for $\sigma = \Theta(1)$.
\end{lemma}
\begin{proof}
Note that for $f(x) = \sgn(x)$, 
$T_{\sigma} f = U_a f$,
where $U$ is the Ornstein-Ulenbeck operator and
$a = 1/\sqrt{1+\sigma^2}$, 
which is $\Omega(1)$ if $\sigma = O(1)$.
For $p$ a polynomial of degree less than $k$, 
with $k$ odd, 
\new{and $h_k$ the degree-$k$ normalized probabilist's Hermite polynomial}, 
we have that
\begin{eqnarray*}
\E_{G \sim \mathcal{N}}[h_k(G)(U_a f-p)(G)]
&=& \E_{G \sim \mathcal{N}}[h_k(G) U_a f(G)] \\
&=&  \E_{G \sim \mathcal{N}}[U_a h_k(G) f(G)] \\
&=& a^k \E_{G \sim \mathcal{N}}[h_k(G) f(G)] \\
&=& 2^{\new{-O(k)}} \E_{G \sim \mathcal{N}}[h_k(G) f(G)] \;,
\end{eqnarray*}
\new{where the first equality uses the orthogonality of Hermite polynomials,
the second equality is by self-adjointness of the OU operator,
the third equality uses Mehler's formula
(namely that $U_a h_k(x) = a^k h_k(x)$), and the fourth follows from
the fact that $1> a = \Omega(1)$.}
\new{Since $k$ is odd,
$|\E_{G \sim \mathcal{N}}[h_k(G) f(G)]| = \Omega(1/\poly(k))$, 
and we get that
\begin{equation} \label{eqn:abs-value}
\left| \E_{G \sim \mathcal{N}}[h_k(G)(U_a f-p)(G)] \right | = 2^{\new{-O(k)}} \;.
\end{equation}
}
Suppose that the polynomial $p$ satisfies $\|U_a f - p\|_1 < \eps < 1$.
Then we have that $\|p\|_1 < 2$. By Fact~\ref{fact:2norm},
this implies that $\|p\|_2 < 2^{k} \|p\|_1 \leq 2^{k+1}$.
Thus, by hypercontractivity (Fact~\ref{fact:hc}), 
we get that
$$\|U_a f - p\|_4 \ll \|U_a f \|_4 + \|p\|_4 = 2^{O(k)} \;.$$
On the other hand, since $|\E[h_k(G)(U_a f(G) - p(G))]| > 2^{-O(k)}$, 
this implies that $$\|U_a f - p\|_2 > 2^{-O(k)} \;.$$
Combined with the above, the Paley-Zygmund inequality 
implies that
$|U_a f(x) - p(x)|> 2^{-O(k)}$ with probability at least $1/2^{O(k)}$
\new{over $x \sim \mathcal{N}$},
which implies  $\|U_a f - p\|_1 > 2^{-O(k)}$.

Thus, if  $\|U_a f(x) - p(x)\|_1 < \eps$, we must have $k \gg \log(1/\eps)$, 
completing the proof. 
\end{proof}

\section{Upper Bound for Smoothed Agnostic Learning under the Gaussian Distribution} \label{sec:alg-Gaussian}

In this section, we prove the following simple proposition, ruling out stronger SQ lower bounds for smoothed agnostic learning of any concept class under Gaussian marginals.

\begin{proposition} \label{prop:alg-g}
Let $(X,y)$ be any distribution on $\R^d \times \{\pm 1\}$ so that $X$ is distributed as a standard Gaussian, and let $1/2 > \sigma, \eps > 0$. Then there is an algorithm that learns $y$ to 0-1 error $\opt_\sigma + \eps$ in $d^{O(\log(1/\eps)/\sigma^2)}$ samples and time.
\end{proposition}
\begin{proof}
We will simply use the $L_1$-polynomial regression algorithm. 
For this, it suffices to prove that for some $m=O(\log(1/\eps)/\sigma^2)$ 
that there is a degree-$m$ polynomial $p$ 
so that $\| p - T_{\sigma} f \|_1 < \eps/2$ where $f(x) = \E[y|X=x]$. 
In fact, we will show that if $p$ is the degree-$m$ part of 
$T_{\sigma} f$ then $\|p-T_{\sigma} f\|_2 < \eps/2$.

We prove this using Hermite analysis. 
In particular, we note that 
$T_\sigma f = U_a g$ where $a=1/\sqrt{1+\sigma^2} = 1-\Omega(\sigma^2)$ 
and $g(x) = f(x/a)$. If we write $g$ in terms of its Hermite expansion, 
we find that $g(x) = \sum_{k=0}^{\infty} g^{[k]}(x)$, 
where $g^{[k]}$ is the degree-$k$ part of the Hermite expansion. 
We have that
$$
\sum_{k=0}^{\infty} \|g^{[k]}(x) \|_2^2 = \|g\|_2^2 \leq 1.
$$
We also have that
$$
T_{\sigma} f = U_a g =  \sum_{k=0}^\infty a^k g^{[k]}.
$$
Letting $p =  \sum_{k=0}^m a^k g^{[k]}(x)$, we have that 
$$
T_{\sigma} f - p =  \sum_{k=m+1}^{\infty} a^k g^{[k]}.
$$
So
$$
\| T_{\sigma} f - p \|_2^2 =  \sum_{k=m+1}^{\infty} a^{2k} \|g^{[k]}\|_2^2 \leq a^{2m}.
$$
Since $a = 1- \Omega(1/\sigma^2)$, if $m$ is at least a sufficiently large constant 
multiple of $\log(1/\eps)/\sigma^2$ this is less than $\eps/2$. This completes the proof. 
\end{proof}

\section{Conclusions} \label{sec:concl}
In this work, we obtained the first SQ lower bounds for smoothed agnostic learning of halfspaces. 
Specifically, 
our SQ lower bounds nearly match known upper bounds for subgaussian distributions. 
Such lower bounds can be naturally extended to various other families 
of single-index models. 
A number of open questions suggest themselves. 
Since a fully-polynomial time algorithm with optimal error may be unachievable in the smoothed 
setting, it would be of interest to develop $\poly(d, 1/\sigma, 1/\eps)$-time 
smoothed agnostic learners with {\em approximate} error guarantees, e.g., error 
$O(\opt_{\sigma})+\eps$ (or give evidence of hardness even for such a weaker guarantee). 
In a different direction, we note that previous upper bounds 
require that the distribution on examples has subexponential tails (or 
better). Is this inherent? Namely, can we prove strong computational 
lower bounds when the input distribution is heavy-tailed? 

\newpage

\bibliographystyle{alpha}
\bibliography{mybib}

\end{document}